\documentclass[conference]{IEEEtran}
\IEEEoverridecommandlockouts
\usepackage{cite}
\usepackage{amsmath,amssymb,amsfonts}
\usepackage{algorithmicx}
\usepackage{algorithm}
\usepackage[noend]{algpseudocode}
\usepackage{amsthm}
\usepackage{bm}
\usepackage{graphicx}
\usepackage{subfigure} 
\usepackage{textcomp}
\usepackage{xcolor}
\usepackage{bbm}
\newtheorem{theorem}{\bf Theorem}
\newtheorem{lemma}{Lemma}
\newtheorem{Remark}{Remark}

\newtheorem{assumption}{\bf Assumption}

\newtheorem{corollary}{\bf Corollary}
\def\BibTeX{{\rm B\kern-.05em{\sc i\kern-.025em b}\kern-.08em
		T\kern-.1667em\lower.7ex\hbox{E}\kern-.125emX}}
\begin{document}
	
	\title{Wyner-Ziv Gradient Compression for Federated Learning
	}
	
	\author{\thanks{This work was supported in part by the National Nature Science Foundation of China (NSFC) under Grant 61901267.
		}\IEEEauthorblockN{Kai Liang\IEEEauthorrefmark{1}\IEEEauthorrefmark{2}\IEEEauthorrefmark{3}, Huiru Zhong\IEEEauthorrefmark{1}\IEEEauthorrefmark{2}\IEEEauthorrefmark{3}, Haoning Chen\IEEEauthorrefmark{1}\IEEEauthorrefmark{2}\IEEEauthorrefmark{3}, and Youlong Wu\IEEEauthorrefmark{1}\\}  
		\IEEEauthorblockA{\IEEEauthorrefmark{1}
			ShanghaiTech University, Shanghai, China}
		\IEEEauthorblockA{\IEEEauthorrefmark{2} Shanghai Institute of Microsystem and Information Technology,
			Chinese Academy of Sciences}
		\IEEEauthorblockA{\IEEEauthorrefmark{3} University of Chinese Academy of Sciences, Beijing, China}
		
		\{liangkai, zhonghr, chenhn, wuyl1\}@shanghaitech.edu.cn
	}

	\maketitle
	
	\begin{abstract}
		Due to limited communication resources at the client and a massive number of model parameters, large-scale distributed learning tasks suffer from communication bottleneck. Gradient compression is an effective method to reduce communication load by transmitting compressed gradients. Motivated by the fact that in the scenario of stochastic gradients descent, gradients between adjacent rounds may have a high correlation since they wish to learn the same model, this paper proposes a \emph{practical} gradient compression scheme for federated learning, which uses historical gradients to compress gradients and is based on Wyner-Ziv coding but without any probabilistic assumption. We also implement our gradient quantization method on the real dataset, and the performance of our method is better than the previous schemes.
	\end{abstract}
	
	
	\section{Introduction}
	{Recent years have witnessed a spurt of progress in modern machine learning technology, more effective and complex machine learning models can be trained through large-scale distributed training.  However, in each iteration of the distributed optimization, information exchange among distributed nodes will incur enormous communication loads  due to the large-scale model parameters. } 
	
	We focus on  federated learning\cite{18}, {which is a distributed learning framework that can effectively help multiple clients perform data usage and machine learning modeling while meeting user privacy protection and data security requirements.} In federated learning clients participating in joint training only need to exchange their own gradients information, without sharing private data.  To alleviate  the communication bottleneck,   gradient compression \cite{1,2,3,4,5,6,7,2020Federated} and efficient mean estimator \cite{8,9,10,11,12,13,14} have  been investigated to  reduce the communication load. {However, most of the previous works on compression fall to exploit any historical gradients at the server. In fact, the historical gradients can be viewed as side information to compress source information, which has been widely studied in classical information theory. For example, \cite{15} first studied the setting of lossy source compression with side information in the decoder. Channel coding can obtain practical coding for distributed source coding \cite{16,17}, but the main bottleneck lies on the expensive computational complexity of coding and decoding. Recently, \cite{19} considered the correlation between the data for distributed mean estimation, and obtained the mean estimator of the dependent variance through the lattice coding scheme. \cite{20} studied distributed mean estimation with side information at the server, and proposed  Wyner-Ziv estimators that require no probabilistic assumption on the clients' data. On the other hand, in convex optimization, the use of historical gradients can accelerate convergence, such as heavy ball method, and Nesterov’s Accelerated gradient Descent (NAG) in \cite{boyd2004convex}. Last but not least, for stochastic gradient descent with variance reduction such as stochastic variance reduced gradient (SVRG)\cite{johnson2013accelerating}, historical gradients can also be used to reduce the variance of stochastic gradients. So an interesting question is whether and how historical gradients can be used for gradient compression.
		
		In this paper, we address this question and propose  \emph{practical} gradient compression schemes for  federated learning by exploring historical gradients as side information. Our contributions are summerized as follows: \emph{1):} {Motivated by Wyner-Ziv estimator of \cite{20}, we propose a local quantized stochastic gradient descent (LQSGD), which exploits the historical gradients as side information to compress the local gradients. Different from \cite{20}, we focus on federated learning with side information instead of distributed mean estimation.  In fact, in the scenario of stochastic gradients descent, gradients between adjacent rounds may have a high correlation since they wish to learn the same model. Therefore, we can use historical gradients to compress gradients.}
		\emph{2):} {We establish an upper bound of the average-squared gradients of our quantization methods. Compared with the case of not using any historical gradients, the average-squared gradients may get a smaller upper bound by using the historical gradients. Besides, We also obtain the convergence rate of our scheme under standard assumptions.}
		\emph{3):} We implement our gradient quantization method and compare it with quantized stochastic gradients descent (QSGD)\cite{3} and the uncompressed SGD on real dataset. We use these three schemes to train linear regression model on datase \texttt{cpusmall-scale}\cite{cpu1} and the classic neural network ResNet18\cite{he2016deep} on \texttt{cifar10}\cite{krizhevsky2009learning} and \texttt{cifar100}\cite{krizhevsky2009learning}. The results show that our method is effective and has better performance compared with QSGD.}

	\section{Problem Setting}
	
	{In this article, we consider the scenario of federated learning in which a set of $N$ clients and one server jointly train a model without exchanging clients' own data, as shown in Fig. 1.} Assume that the clients' communication resources are limited and clients can only communicates with the server. Formally, we deal with the following optimization problem:
	\begin{equation}\label{optproblem}
	\min_{\boldsymbol{\omega}\in\mathbb{R}^d} f(\boldsymbol{\omega})=\frac{1}{N}\sum_{j=1}^{N}f_j(\boldsymbol{\omega}),
	\end{equation}
	where $f_j:\mathbb{R}^d\rightarrow\mathbb{R}$ is the local function corresponding to Client $j$ and $\boldsymbol{\omega}$ is the parameter of model. For instance, in statistical machine learning, $f_j$ is usually the local objective function of each Client $j$ which is the expected loss over the set of data points of node $j$, i.e.,
	\begin{equation}
	f_j(\boldsymbol{\omega})=\mathbb{E}_{\boldsymbol{z}\sim\mathcal{P}_j}l_j(\boldsymbol{\omega},\boldsymbol{z}),
	\end{equation}
	where $\boldsymbol{z}$ is a random variable with probability distribution $\mathcal{P}_j$ and $l_j(\boldsymbol{\omega},\cdot)$ is loss function measuring the performance of the model with parameter $\boldsymbol{\omega}$. Note that the probability distributions of clients may not be necessarily identical. In this paper, we mainly concentrate on homogeneous scenarios where the probability distributions and loss functions of all clients are the same, i.e. $\mathcal{P}_1=\mathcal{P}_2=\cdots=\mathcal{P}_N$ and $l_1=l_2=\cdots=l_N$. {Heterogeneous federated learning with different data distributions and loss functions, of course, is interesting and important. As we will see, the scheme proposed in this paper is also applicable to heterogeneous federated learning, whose analysis is left as future work.}
	\begin{figure}
		\centering
		\includegraphics[width=8cm,height=6.5cm]{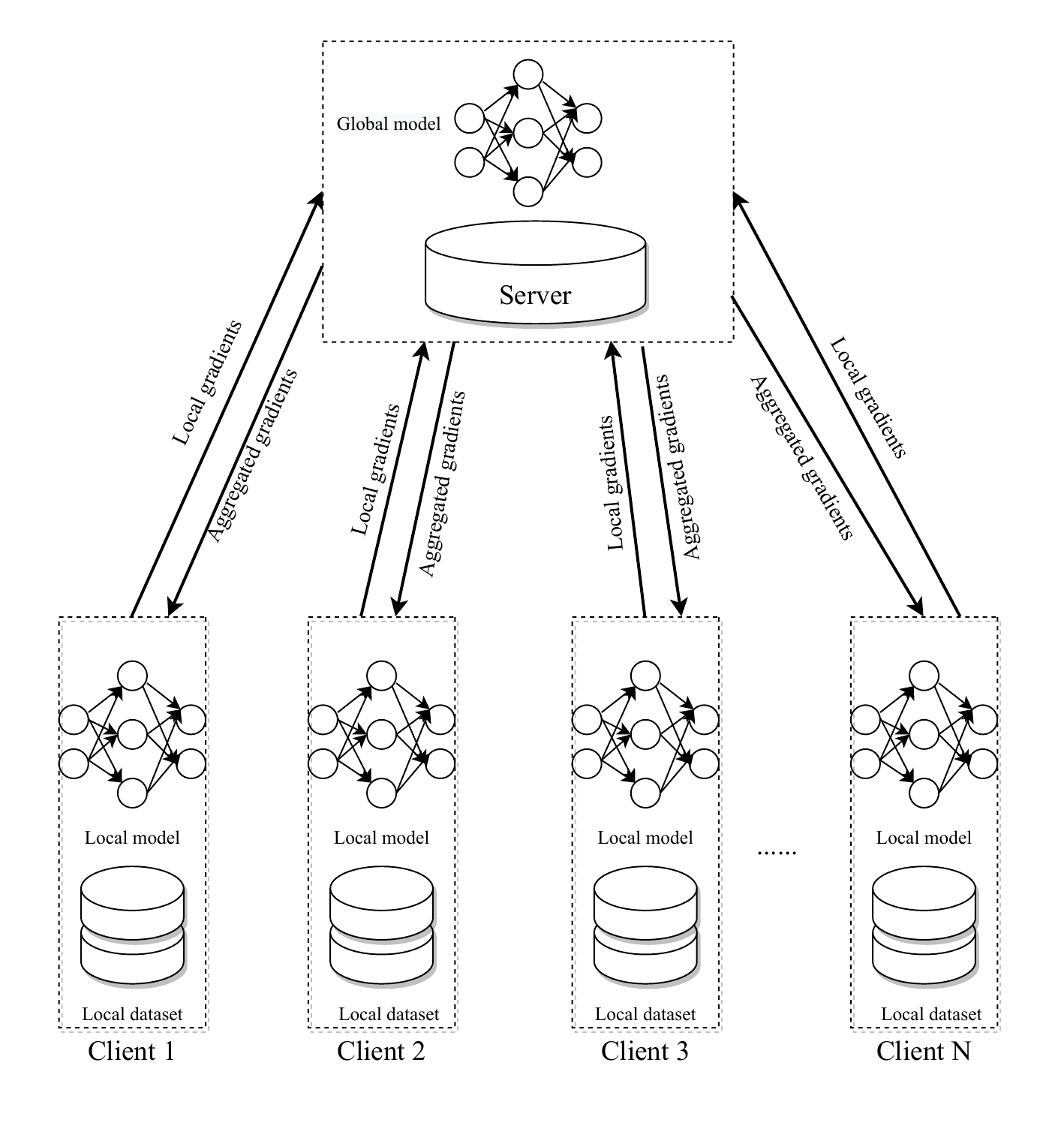} 
		\caption{Federated learning}
		\label{imgModel} 
	\end{figure}
	\section{Federated Aggregation With Side Information}
	In this section, we propose a generalized version of the local quantized stochastic gradient descent (LQSGD) method for federated learning which uses quantized gradients to update model parameters. This method reduces the overall communication overhead by sending fewer bits in each iteration. In Section \ref{fav}, we introduce federated averaging with side information (\texttt{FedSI}) designed for homogeneous settings. Then, in Section \ref{com}, we propose the detailed compression algorithm which uses side information to quantize data. Eventually, we introduce that how to choose the side information used in our method in Section \ref{side}.
	\subsection{Federated Averaging with Side Information(\texttt{FedSI})}\label{fav}
	In federated learning, {the aggregation algorithm is used to make full use of the gradients information of each client}. Each client sends local gradients information to the server after multiple iterations. The server aggregates the gradients sent by clients and then updates the global model and broadcasts it to each client. {Different from the standard federated average algorithm, in \texttt{FedSI}, not only the client sends quantized gradients, but also the global learning rate and local learning rate can be different.}

	Formally, let $R$ be the rounds of communication between server and clients, and $\tau$ be the number of local updates performed between two consecutive communication rounds. Further more, we define $\boldsymbol{\omega}^{(r)}$ as the global model at the master in the $r$-th round of communication. At each round $r$, the server sends the global model $\boldsymbol{\omega}^{(r)}$ to the clients. After that, each Client $j$ computes its local stochastic gradients and updates the model by following the update of SGD
	\begin{equation}
	\boldsymbol{\omega}^{(c+1,r)}=\boldsymbol{\omega}^{(c,r)}-\eta \tilde{g}_j^{(c,r)},\ \text{for}\ c=0,\cdots, \tau-1,
	\end{equation}
	where $\tilde{g}_j^{(c,r)}$ is the estimation of the gradient ${g}_j^{(c,r)}\triangleq\nabla f_j(\boldsymbol{\omega}^{(c,r)})$ and $\eta$ is the leaning rate. For example, for statistical machine learning,  $\tilde{g}_j^{(c,r)}\triangleq\frac{1}{b_j}\sum_{\boldsymbol{z}\in\mathcal{Z}_j^{(c,r)}} \nabla l(\boldsymbol{\omega}^{(c,r)}, \boldsymbol{z})$, where $\mathcal{Z}_j^{(c,r)}$ is the mini-batch of Client $j$ consisting of $b_j$ samples generated by $\mathcal{P}_j$. Next, each client sends the quantized signal $Q(\sum_{c=0}^{\tau-1}\tilde{g}_j^{(c,r)})$ by applying a compression operator $Q(\cdot)$ defined in the Section \ref{com}. When the server receives the signals from all clients, it updates the global model as follows:
	\begin{equation}\label{upd}
	\boldsymbol{\omega}^{(r+1)}=\boldsymbol{\omega}^{(r)}-\frac{\eta\gamma}{N}\sum_{j=1}^{N}Q(\sum_{c=0}^{\tau-1}\tilde{g}_j^{(c,r)}),
	\end{equation}
	where $\gamma$ is the global learing rate. When $\gamma=1$ and all clients send unquantized signals, (\ref{upd}) becomes the standard federated average algorithm.
	\subsection{Compression with Side Information}\label{com}
	In this section, we introduce our compression operator $Q(\cdot)$. The operator $Q(\cdot)$ contains two parts: the encoder $Q_e(\cdot)$ and the decoder $Q_d(\cdot)$, where $Q_e(\cdot):\mathbb{R}^d\rightarrow \{0,1\}^k,\ k\in\mathbb{N}^+$ and $Q_d(\cdot):\{0,1\}^k\rightarrow \mathbb{R}^d,\ k\in\mathbb{N}^+$. 
	
	Now we give a brief introduction about the quantizer $Q_\textnormal{WZ}$\cite{20}, as it is closely related to work.    Since all clients use the same quantizer,  only  the common quantizer is described.  
	We first describe a modulo quantizer $Q_\textnormal{M}$ for one-dimension input $x\in\mathbb{R}$ with side information $h\in\mathbb{R}$, then present a modulo quantizer $Q_{\textnormal{M},d}(x,h)$ for $d$-dimension data.  
	\subsubsection{Modulo Quantizer ($Q_\textnormal{M}$)}\label{cq}  Given the input $x\in\mathbb{R}$ with  side information $h\in\mathbb{R}$, the modulo quantizer $Q_\textnormal{M}$ contains parameters including a distance parameter $\Delta'$ where $|x-h|\leq \Delta'$, a resolution parameter $s\in\mathbb{N}^+$ and a lattice parameter $\epsilon$. 
	
	Denote the encoder and decoder of $Q_{\textnormal{M}}$ as $Q^{\textnormal{e}}_{\textnormal{M}}(x)$ and $Q^{\textnormal{d}}_{\textnormal{M}}(Q^{\textnormal{e}}_{\textnormal{M}}(x),h)$, respectively. 
	The encoder $Q^{\textnormal{e}}_{\textnormal{M}}(x)$ first computes $\lceil x/\epsilon\rceil $ and $\lfloor x/\epsilon\rfloor $, then outputs the message $Q^{\textnormal{e}}_{\textnormal{M}}(x)=m$, where
	\begin{equation}
	m =  \left\{ \begin{array}{llr}
	(\lceil x/\epsilon\rceil\mod s), &~\text{w.p.}~x/\epsilon-\lfloor x/\epsilon\rfloor  \\
	(\lfloor x/\epsilon\rfloor \mod s), &  ~\text{w.p.}~\lceil x/\epsilon\rceil- x/\epsilon
	\end{array}.
	\right.
	\end{equation}
	
	The message $m$ has length of $\log s$ bits, which is  sent  to the decoder.  The decoder $Q^{\textnormal{d}}_{\textnormal{M}}$ produces the estimate $\hat{x}$ by finding a point closest to $h$ in the set $\mathbb{Z}_{m,\epsilon}=\{(zs+m)\cdot\epsilon:z\in\mathbb{Z}\}$. 
	\begin{lemma}{(see \cite{20})}\label{cqpe}
		Consider $Q_\textnormal{M}$ with parameter $\epsilon$ set to satisfy
		\begin{equation}\label{condk}
		s\epsilon\geq2(\epsilon+\Delta').
		\end{equation}
		Then, for every $x,h\in\mathbb{R}$ such that $|x-h|\leq\Delta'$, the output $Q_\textnormal{M}(x)$ satisfies
		\begin{equation*}
		\begin{array}{c}
		\mathbb{E}[Q_\textnormal{M}(x)]=x,\\
		|x-Q_\textnormal{M}(x)|<\epsilon.
		\end{array}
		\end{equation*}
	\end{lemma}
	
	\subsubsection{$d$-dimensional Modulo Quantizer ($Q_{\textnormal{M},d}$)}\label{hcq}
	For d-dimensional data, we use $Q_{\textnormal{M}}$ in each dimension. Specifically, given the input $x\in\mathbb{R}^d$ with  side information $h\in\mathbb{R}^d$, the modulo quantizer $Q_\textnormal{M,d}$ contains parameters including a distance parameter $\Delta'$ where $||x-h||_{\infty}\leq \Delta'$, a resolution parameter $s\in\mathbb{N}^+$ and a lattice parameter $\epsilon$. For the $i$-th dimension data, set $h_i$ to be the side information of $x_i$, where $h_i$ and $x_i$ are the $i$-th data of $h$ and $x$ respectively. For each dimension, we compress the data into $\log s$ bits. Therefore, we use $d\log s$ bits to represent the $d$-dimensional data $x$.
	
	We can easily extend Lemma \ref{cqpe} to $d$-dimensional case, which is given in the following corollary.
	\begin{corollary}\label{co1}
		For $d$-dimensional data, we consider $Q_\textnormal{M,d}$ with parameter $\epsilon$ set to satisfy
		\begin{equation}
		s\epsilon\geq2(\epsilon+\Delta),
		\end{equation}
		Then, for every $\boldsymbol{x},\boldsymbol{y}\in\mathbb{R}^d$ such that $||\boldsymbol{x}-\boldsymbol{y}||_2\leq\Delta$, the output $Q(\boldsymbol{x})$ satisfies
		\begin{equation*}
		\begin{array}{c}
		\mathbb{E}[Q(\boldsymbol{x})]=\boldsymbol{x},\\
		\mathbb{E}||\boldsymbol{x}-Q(\boldsymbol{x})||^2<d\epsilon^2.
		\end{array}
		\end{equation*}
		In addition, if we let $\epsilon=\frac{2\Delta}{s-2}$, we have 
		\begin{equation}\label{var}
		\mathbb{E}||\boldsymbol{x}-Q(\boldsymbol{x})||^2<\frac{4d\Delta^2}{(s-2)^2}.
		\end{equation}
	\end{corollary}
	\begin{Remark}
		In the previous works on mean estimation, such as \cite{8} and \cite{20}, by preprocessing the d-dimensional data, we can get a smaller mean square error. At a high-level, we can rotate the $d$-dimensional data $x$ and $h$ at the same time, for example, $x$ and $h$ are multiplied by  $W = \frac{1}{\sqrt{d}}HD$, where $D$ is a random diagonal matrix with i.i.d. Rademacher entries($\pm1$ with equal probability) and $H$ is a Walsh-Hadamard matrix \cite{21}. After the data is preprocessed, $||Wx-Wh||_\infty$ is much less than $||Wx-Wh||_2=||x-h||_2$ with a high probability. This means that we can estimate the data more accurately while keeping the Euclidean distance unchanged. However, this method requires additional matrix calculations, and the time complexity is $O(d\log d)$. Because we concern about the role of side information in federated learning, for convenience, our implementation does not preprocess the data.
	\end{Remark}
	
	
	\subsection{Select Side Information}\label{side}
	In this part, we discuss which data to be chosen as side information for compressing gradients in federated learning. In convex optimization, historical gradients can accelerate convergence and for stochastic gradient descent with variance reduction, it can also be used to reduce the variance of stochastic gradient. Motivated by the role of historical gradients in accelerating convergence, a natural question is whether the historical gradients can help us compress the current gradients. In fact, gradients between adjacent rounds may have a high correlation since they wish to learn the same model. Therefore, we can regard the historical gradients as the side information of the local gradients.

	{We let $N$ clients share the same side information, which can effectively save memory space for storing side information when the dimension $d$ is large.} Since clients upload quantized gradients, the aggregated gradients broadcasted by the server in the last round may differ greatly from the current local gradients of clients. {In this case, using historical gradients as side information would be worse than the standard QSGD quantization scheme. Therefore, we set a threshold $t$ ($0<t\leq1$), and if the ratio of the distance between the current gradients and historical aggregated gradients to the norm of the current gradients is less than $t$, we use historical gradients as side information, otherwise, we do not use any side information.} Formally, at the $r$-th round of communication, Client $j$ obtains the local gradients $\sum_{c=0}^{\tau-1}\tilde{g}_j^{(c,r)}$ {and calculates}
	\begin{equation}
	D_j^r=\frac{||\sum_{c=0}^{\tau-1}\tilde{g}_j^{(c,r)}-U_q^{r-1}||_2}{||\sum_{c=0}^{\tau-1}\tilde{g}_j^{(c,r)}||_2},
	\end{equation} 
	where $U_q^{r-1}\triangleq\frac{1}{N}\sum_{j=1}^{N}Q(\sum_{c=0}^{\tau-1}\tilde{g}_j^{(c,r-1)})$. If $D_j^r$ is less than $t$, Client $j$ uses historical gradients $U_q^{r-1}$ to compress $\sum_{c=0}^{\tau-1}\tilde{g}_j^{(c,r)}$, otherwise it will not use any side information to compress gradients. For convenience, we use $\alpha_j^r$ to indicate that whether Client $j$ uses historical gradients $U_q^{r-1}$ as side information to compress local gradients at $r$-th round, i.e.,
	\begin{equation}\label{ind}
	\alpha_j^r=  \left\{ \begin{array}{llr}
	1, &D_j^r<t,\\
	0, &  ~\text{otherwise.}
	\end{array}
	\right.
	\end{equation}
	\begin{Remark}
		{Sending the bin number $m$ of each coordinate is a natural way of transmitting messages. However, this naive implementation is sub-optimal. In fact, we can encode the transmitted values by using universal compression schemes\cite{elias1975universal},\cite{apostolico1987robust}. For each client, the bin number $m$ decoded by the server is the same, because lossless compression is used to compress the bin number $m$. Therefore, we use the natural transmission method for convenience, which means that, each client simply transmits $d\log s$ bits in each round of communication.}
	\end{Remark}
	\begin{Remark}
		If we set distance parameter $\Delta'$ as $||x||_2$ and side information $h$ as $0$, our quantizer becomes standard QSGD. Therefore, QSGD can be regarded as a quantizer without any side information. Similarly, if we set distance parameter $\Delta'$ as $||x||_\infty$ and side information $h$ as $0$, our quantizer becomes a variant of QSGD\cite{3} (i.e., QSGDinf\cite{3}, for convenience, we still call it QSGD). Since both the client and the server store side information, we can set the distance parameter $\Delta'$ as $||x-h||_2$ or $||x-h||_\infty$, and it is necessary for the client to send $||x-h||_2$ or $||x-h||_\infty$ to the server. Note that when the dimension $d$ is large enough, the communication cost of sending distance parameter is negligible.
	\end{Remark}

	We formally describe the method in Algorithm~\ref{alg}.
	\begin{algorithm}
		\caption{Federated averaging with side information (\texttt{FedSI})}\label{alg}
		\hspace*{0.02in}{\bf Input:}
		Number of communication rounds $R$, number of local updates $\tau$ , learning rates $\gamma$ and $\eta$, initial global model $\omega^{(0)}$, side information $U_q^{-1}=0$ and threshold $t$
		\begin{algorithmic}[1]
			\For{$0\leq r\leq R-1$}
			\For{each client $j\in[N]$}
			\State Set $\omega_j^{(0,r)}=\omega^{(r)}$
			\For{$0\leq c\leq \tau-1$}
			\State Sample a minibatch $\mathcal{Z}_j^{(c,r)}$ and compute $\tilde{g}_j^{(c,r)}$ 
			\State $\boldsymbol{\omega}^{(c+1,r)}=\boldsymbol{\omega}^{(c,r)}-\eta \tilde{g}_j^{(c,r)}$
			\EndFor
			\State Compute $D_j^r=\frac{||\sum_{c=0}^{\tau-1}\tilde{g}_j^{(c,r)}-U_q^{r-1}||_2}{||\sum_{c=0}^{\tau-1}\tilde{g}_j^{(c,r)}||_2}$
			\State Compute $\alpha_j^r$ according to (\ref{ind})
			\State Client $j$ sends $Q^{e}(\sum_{c=0}^{\tau-1}\tilde{g}_j^{(c,r)})$, $\alpha_j^r$ and $||\sum_{c=0}^{\tau-1}\tilde{g}_j^{(c,r)}-\alpha_j^r U_q^{r-1}||_2$
			\EndFor
			\State Server computes $U_q^r=\frac{1}{N}\sum_{j=1}^{N}Q(\sum_{c=0}^{\tau-1}\tilde{g}_j^{(c,r)})$
			\State Server updates ${\omega}^{(r+1)}=\boldsymbol{\omega}^{(r)}-\eta\gamma U_q^r$ and broadcasts to all clients
			\EndFor  
		\end{algorithmic}	
	\end{algorithm}
	
	\section{Convergence Analysis}
	Next, we present the convergence analysis of our proposed
	method. At the beginning of this section, we state our assumptions which are customary in the analysis of methods with compression (same as \cite{2020Federated}).
	\begin{assumption}\label{ass2}(Smoothness and Lower Boundedness).
		The local objective function $f_j$ of $j$-th client is differentiable for $j\in[m]$ and $L$-smooth, i.e.,
		$||\nabla f_j(\boldsymbol{x})-\nabla f_j(\boldsymbol{y})||\leq L||\boldsymbol{x}-\boldsymbol{y}||$. Moreover,
		the optimal value of objective function $f(\boldsymbol{\omega})$ is
		bounded below by $f^*=\min_{\boldsymbol{\omega}}f(\boldsymbol{\omega})>-\infty$
	\end{assumption}
	\begin{assumption}\label{ass3}
		For all $j\in[N]$, we can sample an independent mini-batch $\mathcal{Z}_j$ of size $|\mathcal{Z}_j^{(c,r)}|=b_j$ and compute an unbiased stochastic gradient $\tilde{g}_j=\nabla f_j(\boldsymbol{\omega};\mathcal{Z}_j); \mathbb{E}_{\mathcal{Z}_j}(\tilde{g}_j)=\nabla f(\boldsymbol{\omega})=g$. Besides, their variance is bounded above by a constant $\sigma^2$,	i.e., $\mathbb{E}||\tilde{g}_j-g||^2\leq \sigma^2$,
	\end{assumption}
	
	In the following theorem, we state our main theoretical results for $\texttt{FedSI}$ in the homogeneous setting.
	\begin{theorem}\label{main}
		For $\texttt{FedSI}$, given the number of communication rounds $R$, the number of local updates $\tau$ , learning rates $\gamma$ and $\eta$, initial global model $\omega^{(0)}$ and threshold $t$, under Assumptions \ref{ass2} and \ref{ass3}, if the learning rate satisfies
		\begin{equation}\label{lea}
		1\geq \tau^2L^2\eta^2+(\frac{q}{N}+1)\tau\gamma L\eta,
		\end{equation}
		where $q=\frac{4d}{(s-2)^2}$, then the average-squared gradients after $R$ communication rounds is bounded as follows:
		\begin{equation}\label{con}
		\begin{aligned}
		\frac{1}{R}\sum_{r=0}^{R-1}||\nabla f(\omega^{(r)})||^2&\leq \frac{2(f(\omega^{(0)})-f(\omega^{(*)}))}{\tau\gamma\eta R}\\
		&\ \!+\!\frac{\gamma L\eta(q\alpha(t\!-\!1)\!+\!q\!+\!1)}{N}\sigma^2\!+\! \tau L^2\eta^2\sigma^2,
		\end{aligned}
		\end{equation}
		where $\alpha\triangleq\frac{1}{RN}\sum_{r,j}\alpha_j^r$.
	\end{theorem}
	\begin{proof}
		See Appendix \ref{pro}.
	\end{proof}
	\begin{Remark}
		Compared with the case where the historical gradients are not used, i.e., $\alpha=0$, our upper bound is reduced by $\frac{\gamma L\eta(q\alpha(1\!-\!t)\!)}{N}\sigma^2$. Obviously, if the learning rates of the two schemes remain the same and $L$, $q$, $\alpha$, and $\sigma^2$ are huge enough, using historical gradients to compress the gradients will converge faster.
	\end{Remark}
	\begin{Remark}
		If we set $\eta\gamma=\Theta(\frac{\sqrt{N}}{\sqrt{R\tau}})$, the upper bound of the average-squared gradients is $O(\frac{1}{\sqrt{NR\tau}}+\frac{1}{R})$. Based on this bound, if we set th communication rounds $R=O(\frac{1}{\delta})$ and the local updates $\tau=O(\frac{1}{N\delta})$,  the average-squared gradients can achieve a $\delta$-accurate.
	\end{Remark}

	\section{Experiments}
	For the simulations, we consider three real datasets: \texttt{cpusmall-scale}\cite{cpu1}, \texttt{cifar10}\cite{krizhevsky2009learning} and \texttt{cifar100}\cite{krizhevsky2009learning}. We implement our gradients quantization method, which we call LQSGD for simplicity and compare it with quantized stochastic gradients descent (QSGD)\cite{3} and uncompressed SGD on real dataset. We use these three schemes to train linear regression model on datase \texttt{cpusmall-scale} and the classic neural network ResNet18\cite{he2016deep} on \texttt{cifar10} and \texttt{cifar100}
	\subsection{Dataset \texttt{cpusmall-scale}}
	The number of samples in the dataset \texttt{cpusmall-scale} is $8192$, and the feature dimension $d$ of each sample is $12$. The classic problem of least-squares is as follows: Some matricex $A$ and target vectors $\boldsymbol{b}$ are given as input and the purpose is to find $\omega^*=\arg\min_{{\omega}}||A\omega-b||_2$. {For convenience, we set the number of local updates $\tau$ as $1$, and simulate the scenarios consisting of eight clients. In order to simulate the scenario where each client's data is sampled from the same distribution, we shuffle the 8192 samples, afterwards assign 1024 samples to each client. We measure the effect on the convergence of the SGD process of the quantization schemes in the next step. After running three algorithms LQSGD, QSGD, and GD, we record the loss and the number of iterations.}
	\begin{figure}
		\centering
		\includegraphics[width=0.5\textwidth]{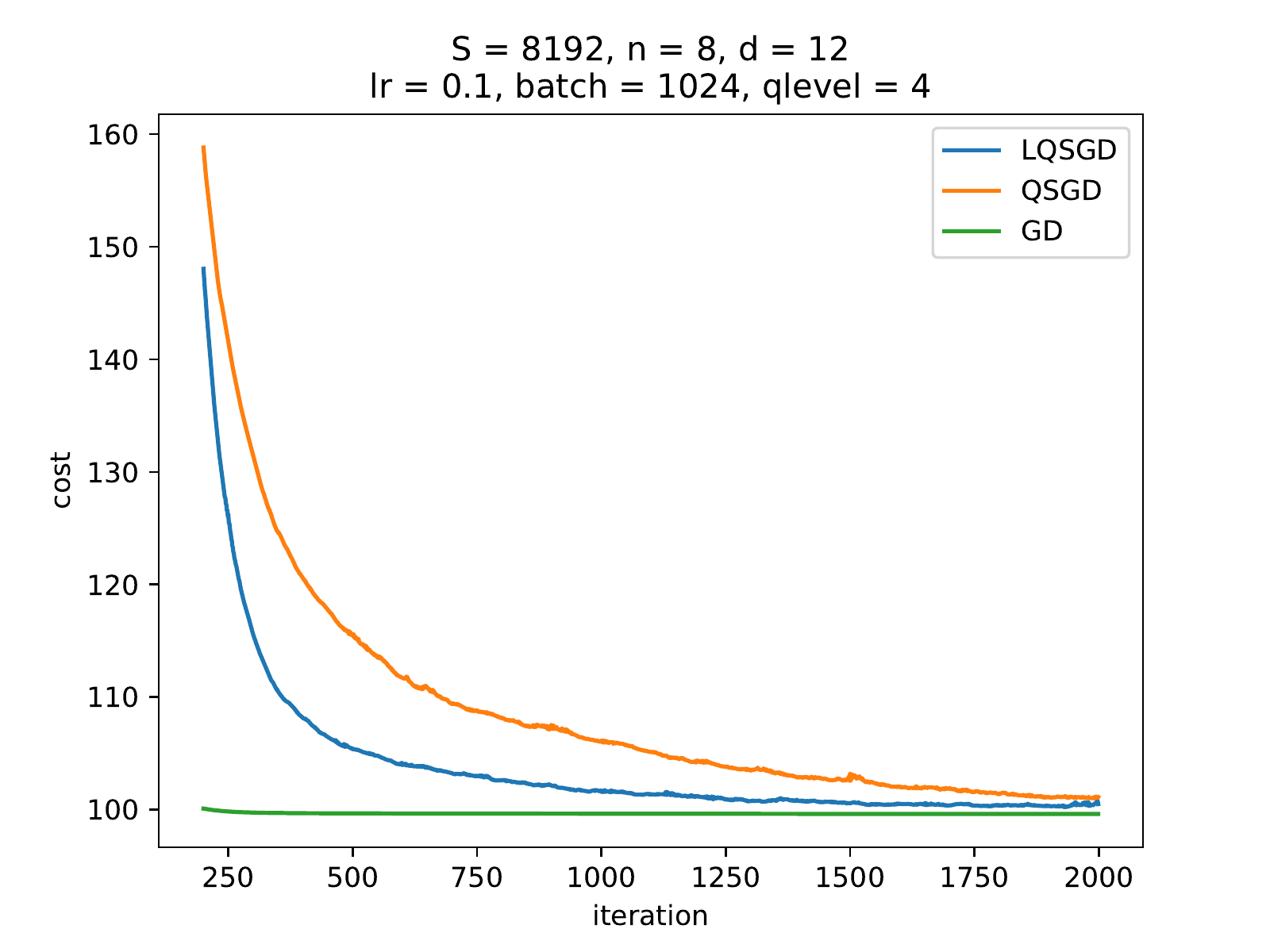} 
		\caption{Convergence at $2$ bits per coordinate}
		\label{lr1} 
	\end{figure}
	
	As shown in Fig. \ref{lr1}, when the learning rate is set as $0.1$ and the value of each coordinate is compressed into $2$ bits, the loss of GD has converged at 250 iterations, the loss of LQSGD decreases faster than QSGD, and LQSGD converges at 1200 iterations, while QSGD converges at 1700 iterations.
	\begin{figure}
		\centering
		\includegraphics[width=0.5\textwidth]{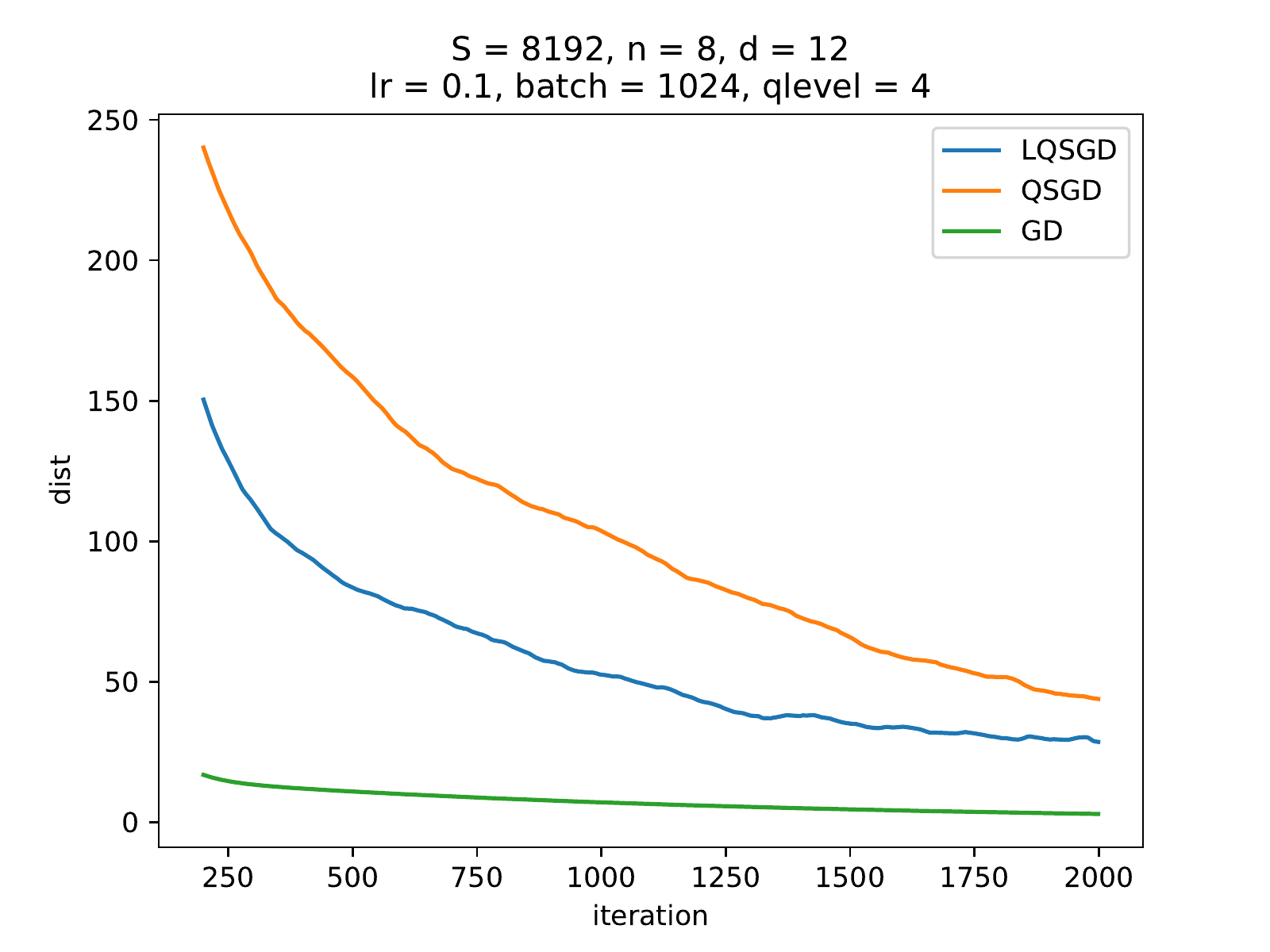} 
		\caption{Distance at $2$ bits per coordinate}
		\label{lr2} 
	\end{figure}
	
	{We also compare the distance (Euclidean distance) among the model parameters obtained by these methods and the optimal model parameters. As shown in Fig. \ref{lr2}, the model parameters obtained by LQSGD are closer to the optimal parameters than model parameters obtained by QSGD.
		\subsection{Dataset \texttt{cifar10} and \texttt{cifar100}}
		\texttt{cifar10} has a total of $60000$ color images of size $32*32$, and these images are divided into $10$ categories, each of them with $6000$ images. There are $50000$ images used for training, forming $5$ training batches with $10000$ images, the other $10000$ are used for testing, forming a single batch. The data of the test batch is taken from each of the $10$ categories, and $1000$ images are randomly taken from each category. The rest is randomly arranged to form a training batch. \texttt{cifar100} is similar to \texttt{cifar10}, except that it has $100$ classes and each class contains $600$ images. Every category has $500$ training images and $100$ test images. We also simulate the scenario of $8$ clients and randomly assign training images to eight clients to ensure that the number of images in each class is the same for each client. We train the classic neural network ResNet18 on datasets. It is impractical to calculate the $2$-norm of the gradients because the number of parameters of ResNet18 is about $33$ million. In practice, we used the $\infty$-norm of gradients in LQSGD and QSGD.}
	\begin{table}[h]
		\centering 
		\caption{Test accuracy on \texttt{cifar10} and \texttt{cifar100} using $3$ bits (except for SGD) with $8$ clients} 
		\begin{tabular}{|p{2.5cm}|p{2.5cm}|p{2.5cm}|}   
			\hline  
			\hline  
			& \texttt{cifar10} & \texttt{cifar100} \\ 
			\hline
			SGD & $\bm{94.61\%}$& 76.44\%\\  
			\hline  
			QSGD& 94.19\% & 76.03\%\\  
			\hline  
			LQSGD &94.53\%&$\bm{76.45\%}$\\  
			\hline  
			\hline  
		\end{tabular}  
		\label{ta}
	\end{table} 
	We train $100$ epoches of the model ResNet18 and  evaluate the models obtained by using LQSGD, QSGD, and SGD on the testset. The test accuracy is shown in Table \ref{ta}. As shown in Table \ref{ta}, on the dataset \texttt{cifar10}, compared with SGD, the test accuracy of QSGD loses $0.42\%$, while the test accuracy of LQSGD loses $0.08\%$. Similarly, on the dataset \texttt{cifar100}, compared with SGD, the test accuracy of QSGD loses $0.41\%$, but the test accuracy of LQSGD is almost the same as that of SGD. It should be noted that the test accuracy of LQSGD is higher than the test accuracy of SGD and this is because our quantization scheme is an unbiased estimate of the gradients, which means that we add additional noise with a mean value of $0$ to the gradients. This may cause our test accuracy to be higher than that of SGD.
	
	We also record the relationship between the test accuracy and the number of epoches. The details are shown in Fig. \ref{nn1} and \ref{nn2}. As shown in Fig. \ref{nn1} and \ref{nn2}, After about $70$ epoches, the test accuracy curves of LQSGD and SGD are almost above the test accuracy curve of QSGD. This means that the performance of our method is better than that of QSGD after about $70$ epoches.
	
	\begin{figure}
		\centering
		\includegraphics[width=0.5\textwidth]{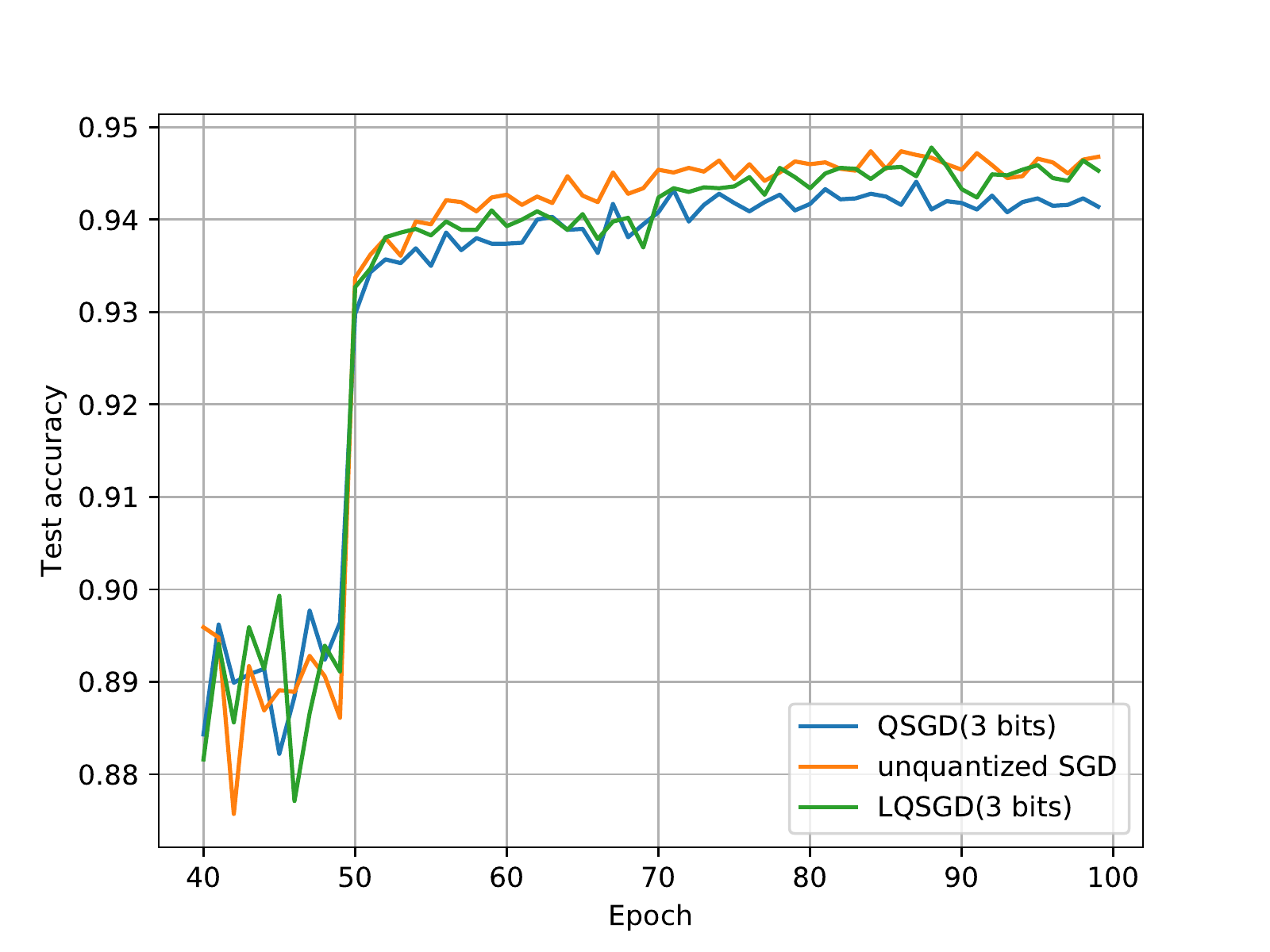} 
		\caption{Test accuracy on \texttt{cifar10} for $8$ clients}
		\label{nn1} 
	\end{figure}
	\begin{figure}
		\centering
		\includegraphics[width=0.5\textwidth]{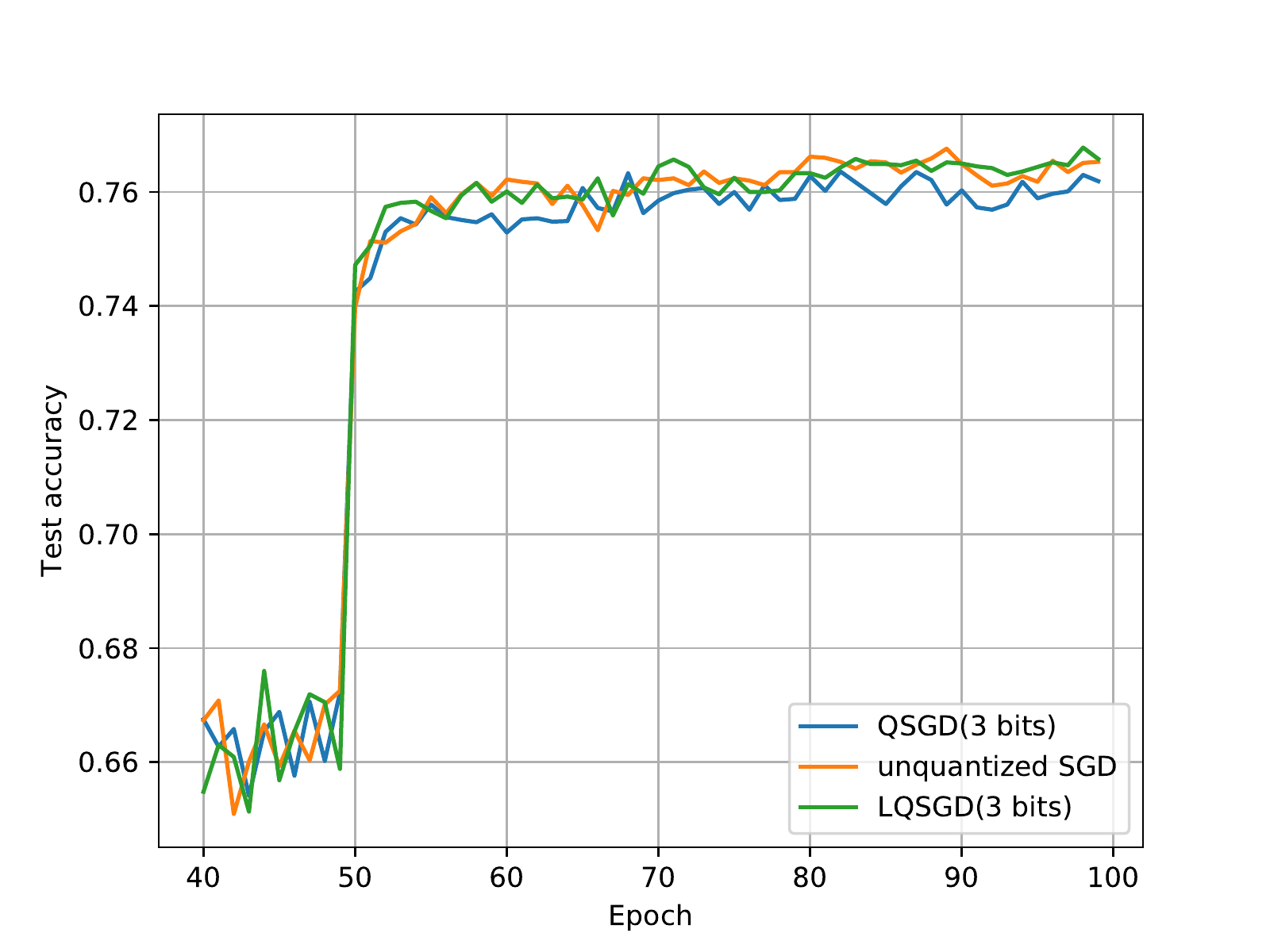} 
		\caption{Test accuracy on \texttt{cifar100} for $8$ clients}
		\label{nn2} 
	\end{figure}
	\section{Conclusions and Future works}
	In this paper, we proposed a gradients quantization method for federated learning, which uses the historical gradients as side information to compress the local gradients. We gave an upper bound of the average-squared gradients of quantization methods and also proved the convergence rate of our scheme under standard assumptions. We  not only implemented our gradients quantization method, but also demonstrated the superiority of our scheme over previous work QSGD empirically on deep models and linear regression. In future work, we will extend our scheme to heterogeneous settings and look for ways to use side information more efficiently in federated learning.
	
	\bibliographystyle{IEEEtran}
	\bibliography{1.bib}

\begin{thebibliography}{10}
\providecommand{\url}[1]{#1}
\csname url@samestyle\endcsname
\providecommand{\newblock}{\relax}
\providecommand{\bibinfo}[2]{#2}
\providecommand{\BIBentrySTDinterwordspacing}{\spaceskip=0pt\relax}
\providecommand{\BIBentryALTinterwordstretchfactor}{4}
\providecommand{\BIBentryALTinterwordspacing}{\spaceskip=\fontdimen2\font plus
\BIBentryALTinterwordstretchfactor\fontdimen3\font minus
  \fontdimen4\font\relax}
\providecommand{\BIBforeignlanguage}[2]{{%
\expandafter\ifx\csname l@#1\endcsname\relax
\typeout{** WARNING: IEEEtran.bst: No hyphenation pattern has been}%
\typeout{** loaded for the language `#1'. Using the pattern for}%
\typeout{** the default language instead.}%
\else
\language=\csname l@#1\endcsname
\fi
#2}}
\providecommand{\BIBdecl}{\relax}
\BIBdecl

\bibitem{18}
P.~Kairouz, H.~B. McMahan, B.~Avent, A.~Bellet, M.~Bennis, A.~N. Bhagoji,
  K.~Bonawitz, Z.~Charles, G.~Cormode, R.~Cummings \emph{et~al.}, ``Advances
  and open problems in federated learning,'' \emph{arXiv preprint
  arXiv:1912.04977}, 2019.

\bibitem{1}
F.~Seide, H.~Fu, J.~Droppo, G.~Li, and D.~Yu, ``1-bit stochastic gradient
  descent and its application to data-parallel distributed training of speech
  dnns,'' in \emph{Fifteenth Annual Conference of the International Speech
  Communication Association}, 2014.

\bibitem{2}
H.~Wang, S.~Sievert, S.~Liu, Z.~Charles, D.~Papailiopoulos, and S.~Wright,
  ``Atomo: Communication-efficient learning via atomic sparsification,''
  \emph{Advances in Neural Information Processing Systems}, vol.~31, pp.
  9850--9861, 2018.

\bibitem{3}
D.~Alistarh, D.~Grubic, J.~Li, R.~Tomioka, and M.~Vojnovic, ``Qsgd:
  Communication-efficient sgd via gradient quantization and encoding,'' in
  \emph{Advances in Neural Information Processing Systems}, 2017, pp.
  1709--1720.

\bibitem{4}
D.~Alistarh, T.~Hoefler, M.~Johansson, N.~Konstantinov, S.~Khirirat, and
  C.~Renggli, ``The convergence of sparsified gradient methods,'' in
  \emph{Advances in Neural Information Processing Systems}, 2018, pp.
  5973--5983.

\bibitem{5}
S.~U. Stich, J.-B. Cordonnier, and M.~Jaggi, ``Sparsified sgd with memory,'' in
  \emph{Advances in Neural Information Processing Systems}, 2018, pp.
  4447--4458.

\bibitem{6}
W.~Wen, C.~Xu, F.~Yan, C.~Wu, Y.~Wang, Y.~Chen, and H.~Li, ``Terngrad: Ternary
  gradients to reduce communication in distributed deep learning,'' in
  \emph{Advances in neural information processing systems}, 2017, pp.
  1509--1519.

\bibitem{7}
J.~Wangni, J.~Wang, J.~Liu, and T.~Zhang, ``Gradient sparsification for
  communication-efficient distributed optimization,'' \emph{Advances in Neural
  Information Processing Systems}, vol.~31, pp. 1299--1309, 2018.

\bibitem{2020Federated}
F.~Haddadpour, M.~M. Kamani, A.~Mokhtari, and M.~Mahdavi, ``Federated learning
  with compression: Unified analysis and sharp guarantees,'' 2020.

\bibitem{8}
A.~T. Suresh, X.~Y. Felix, S.~Kumar, and H.~B. McMahan, ``Distributed mean
  estimation with limited communication,'' in \emph{International Conference on
  Machine Learning}.\hskip 1em plus 0.5em minus 0.4em\relax PMLR, 2017, pp.
  3329--3337.

\bibitem{9}
J.~Kone{\v{c}}n{\`y} and P.~Richt{\'a}rik, ``Randomized distributed mean
  estimation: Accuracy vs. communication,'' \emph{Frontiers in Applied
  Mathematics and Statistics}, vol.~4, p.~62, 2018.

\bibitem{10}
W.-N. Chen, P.~Kairouz, and A.~{\"O}zg{\"u}r, ``Breaking the
  communication-privacy-accuracy trilemma,'' \emph{arXiv preprint
  arXiv:2007.11707}, 2020.

\bibitem{11}
Z.~Huang, W.~Yilei, K.~Yi \emph{et~al.}, ``Optimal sparsity-sensitive bounds
  for distributed mean estimation,'' in \emph{Advances in Neural Information
  Processing Systems}, 2019, pp. 6371--6381.

\bibitem{12}
P.~Mayekar and H.~Tyagi, ``Ratq: A universal fixed-length quantizer for
  stochastic optimization,'' in \emph{International Conference on Artificial
  Intelligence and Statistics}.\hskip 1em plus 0.5em minus 0.4em\relax PMLR,
  2020, pp. 1399--1409.

\bibitem{13}
M.~Safaryan, E.~Shulgin, and P.~Richt{\'a}rik, ``Uncertainty principle for
  communication compression in distributed and federated learning and the
  search for an optimal compressor,'' \emph{arXiv preprint arXiv:2002.08958},
  2020.

\bibitem{14}
A.~Albasyoni, M.~Safaryan, L.~Condat, and P.~Richt{\'a}rik, ``Optimal gradient
  compression for distributed and federated learning,'' \emph{arXiv preprint
  arXiv:2010.03246}, 2020.

\bibitem{15}
A.~Wyner and J.~Ziv, ``The rate-distortion function for source coding with side
  information at the decoder,'' \emph{IEEE Transactions on information Theory},
  vol.~22, no.~1, pp. 1--10, 1976.

\bibitem{16}
S.~S. Pradhan and K.~Ramchandran, ``Distributed source coding using syndromes
  (discus): Design and construction,'' \emph{IEEE transactions on information
  theory}, vol.~49, no.~3, pp. 626--643, 2003.

\bibitem{17}
R.~Zamir, S.~Shamai, and U.~Erez, ``Nested linear/lattice codes for structured
  multiterminal binning,'' \emph{IEEE Transactions on Information Theory},
  vol.~48, no.~6, pp. 1250--1276, 2002.

\bibitem{19}
P.~Davies, V.~Gurunathan, N.~Moshrefi, S.~Ashkboos, and D.~Alistarh,
  ``Distributed variance reduction with optimal communication,'' \emph{arXiv
  e-prints}, pp. arXiv--2002, 2020.

\bibitem{20}
P.~Mayekar, A.~T. Suresh, and H.~Tyagi, ``Wyner-ziv estimators: Efficient
  distributed mean estimation with side information,'' \emph{arXiv preprint
  arXiv:2011.12160}, 2020.

\bibitem{boyd2004convex}
S.~Boyd, S.~P. Boyd, and L.~Vandenberghe, \emph{Convex optimization}.\hskip 1em
  plus 0.5em minus 0.4em\relax Cambridge university press, 2004.

\bibitem{johnson2013accelerating}
R.~Johnson and T.~Zhang, ``Accelerating stochastic gradient descent using
  predictive variance reduction,'' \emph{Advances in neural information
  processing systems}, vol.~26, pp. 315--323, 2013.

\bibitem{cpu1}
C.-J.~L. Chih-Chung~Chang, ``Libsvm datasets,''
  https://www.csie.ntu.edu.tw/~cjlin/libsvmtools/datasets/.

\bibitem{he2016deep}
K.~He, X.~Zhang, S.~Ren, and J.~Sun, ``Deep residual learning for image
  recognition,'' in \emph{Proceedings of the IEEE conference on computer vision
  and pattern recognition}, 2016, pp. 770--778.

\bibitem{krizhevsky2009learning}
A.~Krizhevsky, G.~Hinton \emph{et~al.}, ``Learning multiple layers of features
  from tiny images,'' 2009.

\bibitem{21}
K.~J. Horadam, \emph{Hadamard matrices and their applications}.\hskip 1em plus
  0.5em minus 0.4em\relax Princeton university press, 2012.

\bibitem{elias1975universal}
P.~Elias, ``Universal codeword sets and representations of the integers,''
  \emph{IEEE Transactions on Information Theory}, vol.~21, no.~2, pp. 194--203,
  1975.

\bibitem{apostolico1987robust}
A.~Apostolico and A.~Fraenkel, ``Robust transmission of unbounded strings using
  fibonacci representations,'' \emph{IEEE Transactions on Information Theory},
  vol.~33, no.~2, pp. 238--245, 1987.

\end{thebibliography}
	\newpage
	\appendices
	\section{Proof of Theorem \ref{main}}\label{pro}
	Our proof method is similar to the method in \cite{2020Federated}, and in our proof we use the intermediate results in \cite{2020Federated}.
	Before stating the proof of  Theorem \ref{main}, we first give the following lemma.
	\begin{lemma}\label{l1}
		Under Assumptions \ref{ass3} and Corollary \ref{co1}, we have the following bound:
		\begin{equation}
		\begin{aligned}
		\mathbb{E}_{Q,\boldsymbol{z}}&\big|\big|\frac{1}{N}\sum_{j=1}^{N}Q(\sum_{c=0}^{\tau-1}\tilde{g}_j^{(c,r)})\big|\big|^2\\
		&\leq \frac{\tau\sigma^2}{N}(\frac{q}{N}\sum_{j=1}^{N}(\alpha_j^r(t-1)+1)+1)\\
		&\quad\!+\!\frac{\tau}{N}\sum_{j=1}^{N}(\frac{q(\alpha_j^r(t-1)+1)}{N}\!+\!1)\sum_{c=0}^{\tau-1}||g_j^{(c,r)}||^2.
		\end{aligned}
		\end{equation}
	\end{lemma}
	\begin{proof}
		For convenience, let $\tilde{g}_j^{(r)}\triangleq\sum_{c=0}^{\tau-1}\tilde{g}_j^{(c,r)}$, $\tilde{g}_{Qj}^{(r)}\triangleq Q(\tilde{g}_j^{(r)})$ and $\tilde{g}_Q^{(r)}\triangleq\frac{1}{N}\sum_{j=1}^{N}\tilde{g}_{Qj}^{(r)}$. Then, we have
		\begin{equation}
		\begin{aligned}
		\mathbb{E}&_{Q,\boldsymbol{z}}\big|\big|\tilde{g}_Q^{(r)}\big|\big|^2\\
		&=\mathbb{E}_{\boldsymbol{z}}\bigg[\mathbb{E}_{Q}\big[\big|\big|\tilde{g}_Q^{(r)}\!-\!\mathbb{E}_Q(\tilde{g}_Q^{(r)})\big|\big|^2\big]\!+\!\big|\big|\mathbb{E}_Q(\tilde{g}_Q^{(r)})\big|\big|^2\bigg]\\
		&=\mathbb{E}_{\boldsymbol{z}}\bigg[\mathbb{E}_{Q}\big[\frac{1}{N^2}\sum_{j=1}^{N}\big|\big|\tilde{g}_{Qj}^{(r)}-\tilde{g}_{j}^{(r)}\big|\big|^2\big]+\big|\big|\frac{1}{N}\sum_{j=1}^{N}\tilde{g}_{j}^{(r)})\big|\big|^2\bigg]\\
		&\overset{(a)}{\leq}\mathbb{E}_{\boldsymbol{z}}\bigg[\big[\frac{q}{N}\sum_{j=1}^{N}(\alpha_j^r(D_j^r-1)+1)\big|\big|\tilde{g}_{j}^{(r)}\big|\big|^2\big]+\big|\big|\frac{1}{N}\sum_{j=1}^{N}\tilde{g}_{j}^{(r)}\big|\big|^2\bigg]\\
		&=\frac{q}{N^2}\sum_{j=1}^{N}(\alpha_j^r(D_j^r-1)+1)\mathbb{E}_{\boldsymbol{z}}\big|\big|\tilde{g}_{j}^{(r)}\big|\big|^2+\mathbb{E}_{\boldsymbol{z}}\big|\big|\frac{1}{N}\sum_{j=1}^{N}\tilde{g}_{j}^{(r)}\big|\big|^2\\
		&=\frac{q}{N^2}\sum_{j=1}^{N}(\alpha_j^r(D_j^r-1)+1)\big(\mathbb{E}_{\boldsymbol{z}}\big|\big|\tilde{g}_{j}^{(r)}-{g}_{j}^{(r)}\big|\big|^2+\big|\big|{g}_{j}^{(r)}\big|\big|^2\big)\\
		&\quad+\mathbb{E}_{\boldsymbol{z}}\big|\big|\frac{1}{N}\sum_{j=1}^{N}\tilde{g}_{j}^{(r)}-\frac{1}{N}\sum_{j=1}^{N}{g}_{j}^{(r)}\big|\big|^2+\big|\big|\frac{1}{N}\sum_{j=1}^{N}{g}_{j}^{(r)}\big|\big|^2\\
		&\overset{(b)}{\leq}\frac{q}{N^2}\sum_{j=1}^{N}(\alpha_j^r(D_j^r\!-\!1)\!+\!1)\big(\tau\sigma^2\!+\!\big|\big|{g}_{j}^{(r)}\big|\big|^2\big)\!+\!\frac{\tau\sigma^2}{N}\!+\!\frac{1}{N}\sum_{j=1}^{N}\big|\big|{g}_{j}^{(r)}\big|\big|^2\\
		&\leq\frac{\tau\sigma^2}{N}(\frac{q}{N}\sum_{j=1}^{N}(\alpha_j^r(t-1)+1)+1)\\
		&\quad\!+\!\frac{1}{N}\sum_{j=1}^{N}(\frac{q(\alpha_j^r(t-1)+1)}{N}\!+\!1)\big|\big|{g}_{j}^{(r)}\big|\big|^2\\
		&\leq \frac{\tau\sigma^2}{N}(\frac{q}{N}\sum_{j=1}^{N}(\alpha_j^r(t-1)+1)+1)\\
		&\quad\!+\!\frac{\tau}{N}\sum_{j=1}^{N}(\frac{q(\alpha_j^r(t-1)+1)}{N}\!+\!1)\sum_{c=0}^{\tau-1}||g_j^{(c,r)}||^2,
		\end{aligned}
		\end{equation}
		where $(a)$ follows by the Corollary \ref{co1} and the definition of $\alpha_j^r$, and $(b)$ holds by the assumption \ref{ass3} and the fact that the sampling process of different clients is independent. In addition, we used the fact that for any random variable $\boldsymbol{X}$,  $\mathbb{E}||X||^2=\mathbb{E}||X-\mathbb{E}X||^2+||\mathbb{E}X||^2$.
	\end{proof}
	The following two lemmas are used in our proof.
	\begin{lemma}[See \cite{2020Federated}]\label{l2}
		Under Assumption \ref{ass2}, the expected inner product between stochastic gradient and full batch gradient can be bounded with:
		\begin{equation}
		\begin{aligned}
		-\mathbb{E}[\!<\!\nabla f(\omega^{(r)}), \eta\tilde{g}_Q^{(r)}\!\!>]&\leq\frac{\eta}{2N}\sum_{j,c}\bigg[\!-\!||\nabla f(\omega^{(r)})||^2\!-\!||\nabla f(\omega_j^{(c,r)})||^2\\
		&\quad+L^2\mathbb{E}||\omega^{(r)}-\omega_j^{(c,r)}||^2\bigg].
		\end{aligned}
		\end{equation}
	\end{lemma}
	\begin{lemma}[See \cite{2020Federated}]\label{l3}
		Under Assumption \ref{ass3}, we have:
		\begin{equation}
		\mathbb{E}\big[||\omega^{(r)}-\omega_j^{(c,r)}||^2\big]\leq\eta^2\tau(\sum_{c=0}^{\tau-1}||g_j^{(c,r)}||^2+\sigma^2).
		\end{equation}
	\end{lemma}

	From the $L$-smoothness gradient assumption on $f$, we have:
	\begin{equation}\label{up1}
	f(\omega^{(r+1)})-f(\omega^{(r)})\leq-\gamma\!<\!\nabla f(\omega^{(r)}), \eta\tilde{g}_Q^{(r)}\!\!>+\frac{\gamma^2\eta^2L}{2}||\tilde{g}_Q^{(r)}||^2
	\end{equation}
	By taking expectation on both sides of (\ref{up1}) over sampling, we get:
	\begin{equation}\label{oneit}
	\begin{aligned}
	\mathbb{E}&\bigg[f(\omega^{(r+1)})-f(\omega^{(r)})\bigg]\\
	&\leq\mathbb{E}\big[-\gamma\!<\!\nabla f(\omega^{(r)}), \eta\tilde{g}_Q^{(r)}\!\!>\big]+\frac{\gamma^2\eta^2L}{2}\mathbb{E}||\tilde{g}_Q^{(r)}||^2\\
	&\overset{(a)}{\leq} \frac{\eta\gamma}{2N}\sum_{j,c}\bigg[\!-\!||\nabla f(\omega^{(r)})||^2\!-\!||\nabla f(\omega_j^{(c,r)})||^2\\
	&\quad+L^2\eta^2\tau(\sum_{c=0}^{\tau-1}||g_j^{(c,r)}||^2+\sigma^2)\bigg]\\
	&\quad+\frac{\gamma^2\eta^2L}{2}\bigg[\frac{\tau\sigma^2}{N}(\frac{q}{N}\sum_{j=1}^{N}(\alpha_j^r(t-1)+1)+1)\\
	&\quad\!+\!\frac{\tau}{N}\sum_{j=1}^{N}(\frac{q(\alpha_j^r(t-1)+1)}{N}\!+\!1)\sum_{c=0}^{\tau-1}||g_j^{(c,r)}||^2\bigg]\\
	&=-\frac{\eta\gamma\tau}{2}||\nabla f(\omega^{(r)})||^2-\bigg(1-\tau^2L^2\eta^2\\
	&\quad-(\frac{\sum_{j=1}^{N}q(\alpha_j^r(t-1)+1)}{N^2}+1)\tau\gamma L\eta\bigg)\frac{\eta\gamma}{2N}\sum_{j=1}^{N}\sum_{c=0}^{\tau-1}||g_j^{(c,r)}||^2\\
	&\quad+\frac{L\tau\gamma\eta^2}{2N}[NL\eta\tau+\gamma[q(\alpha_j^r(t-1)+1)+1]]\sigma^2\\
	&\overset{(b)}{\leq}-\frac{\eta\gamma\tau}{2}||\nabla f(\omega^{(r)})||^2\\
	&\quad+\frac{L\tau\gamma\eta^2}{2N}[NL\eta\tau+\gamma[\frac{\sum_{j=1}^{N}q(\alpha_j^r(t-1)+1)}{N}+1]]\sigma^2,\\
	\end{aligned}
	\end{equation}
	where $(a)$ follows by the Lemma \ref{l1}, Lemma \ref{l2} and Lemma \ref{l3}, and $(b)$ holds by the following condition
	\begin{equation}
	1\geq \tau^2L^2\eta^2+(\frac{q}{N}+1)\tau\gamma L\eta,
	\end{equation}
	and the fact $\frac{\sum_{j=1}^{N}q(\alpha_j^r(t-1)+1)}{N}\leq q$.

	Summing up (\ref{oneit}) for all $R$ communication rounds and rearranging the terms gives:
	\begin{equation}
	\begin{aligned}
	\frac{1}{R}\sum_{r=0}^{R-1}||\nabla f(\omega^{(r)})||^2&\leq \frac{2(f(\omega^{(0)})-f(\omega^{(*)}))}{\tau\gamma\eta R}\\
	&\ \!+\!\frac{\gamma L\eta(q\alpha(t\!-\!1)\!+\!q\!+\!1)}{N}\sigma^2\!+\! \tau L^2\eta^2\sigma^2,
	\end{aligned}
	\end{equation}
	where $\alpha\triangleq\frac{1}{RN}\sum_{r,j}\alpha_j^r$. We completed the proof.

\end{document}